\documentclass[11pt]{article}
\usepackage{amsmath,amssymb,amsfonts,amsthm,epsfig}
\usepackage[usenames,dvipsnames]{xcolor}
\usepackage{bm,xspace}
\usepackage{tcolorbox}
\usepackage{cancel}
\usepackage{fullpage}
\usepackage[backref, colorlinks,citecolor=blue,linkcolor=magenta,bookmarks=true]{hyperref}  
\usepackage{liyang}
\usepackage{framed}
\usepackage{verbatim}
\usepackage{enumitem}
\usepackage{array}
\usepackage{multirow}
\usepackage{afterpage}
\usepackage{mathrsfs}
\usepackage{pifont} 
\usepackage{chngpage}
\usepackage[normalem]{ulem}
\usepackage{boxedminipage}
\usepackage{caption}

\def\colorful{0}

\ifnum\colorful=1
\newcommand{\violet}[1]{{\color{violet}{#1}}}

\newcommand{\blue}[1]{{{\color{blue}#1}}}

\fi
\ifnum\colorful=0
\newcommand{\violet}[1]{{{#1}}}

\newcommand{\blue}[1]{{{#1}}}

\fi

\newcommand{\Gimpurity}{\mathscr{G}\text{-}\mathrm{impurity}}

\newcommand{\pmo}{\{\pm 1\}}

\newcommand{\puritygainS}{\mathrm{PurityGain}_{\mathscr{G},S}}
\newcommand{\puritygainf}{\mathrm{PurityGain}_{\mathscr{G},f}}
\newcommand{\localgainS}{\mathrm{LocalGain}_{\mathscr{G},S}}
\newcommand{\localgainf}{\mathrm{LocalGain}_{\mathscr{G},f}}
\newcommand{\estpuritygainf}{\mathrm{EstPurityGain}_{\mathscr{G},f}}
\newcommand{\estlocalgainf}{\mathrm{EstLocalGain}_{\mathscr{G},f}}
\newcommand{\estpuritygainS}{\mathrm{EstPurityGain}_{\mathscr{G},S}}
\newcommand{\estlocalgainS}{\mathrm{EstLocalGain}_{\mathscr{G},S}}
\newcommand{\Batch}{\mathrm{Batch}}

\newcommand{\LocalLearner}{\textsc{LocalLearner}}
\newcommand{\TopDownSizeEstimate}{\textsc{TopDownSizeEstimate}}

\makeatletter
\newtheorem*{rep@theorem}{\rep@title}
\newcommand{\newreptheorem}[2]{
\newenvironment{rep#1}[1]{
 \def\rep@title{#2 \ref{##1}}
 \begin{rep@theorem}\itshape}
 {\end{rep@theorem}}}
\makeatother
\theoremstyle{plain}

\newreptheorem{theorem}{Theorem}

\makeatletter
\newtheorem*{rep@claim}{\rep@title}
\newcommand{\newrepclaim}[2]{
\newenvironment{rep#1}[1]{
 \def\rep@title{#2 \ref{##1}}
 \begin{rep@claim}\itshape}
 {\end{rep@claim}}}
\makeatother
\theoremstyle{plain}

\newrepclaim{claim}{Claim}

\begin{document}

\newcommand{\error}{\mathrm{error}}
\newcommand{\Est}{\textsc{Est}}
\renewcommand{\bn}{\pmo^n}
\renewcommand{\bits}{\pmo}
\newcommand{\TopDown}{\textsc{TopDown}}
\newcommand{\Score}{\mathrm{Score}}

\newcommand{\MiniBatchTopDown}{\textsc{MiniBatchTopDown}}
\newcommand{\MiniBatchTopDownCapped}{\textsc{MiniBatchTopDownCapped}}

\title{Estimating decision tree learnability \\ with polylogarithmic sample complexity\vspace{15pt}}

\author{%
  Guy Blanc \vspace{5pt} \\
  \hspace{-5pt}{\sl Stanford} 
      \and
  Neha Gupta \vspace{5pt} \\
  \hspace{-5pt}{\sl Stanford} 
  \and
  Jane Lange \vspace{5pt} \\
  \hspace{-5pt}{\sl MIT} 
  \and
    Li-Yang Tan \vspace{5pt} \\
  \hspace{-5pt}{\sl Stanford} 
  \vspace{15pt} 
  }
\date{\small{\today}}

\maketitle

\begin{abstract}
We show that top-down decision tree learning heuristics are amenable to highly efficient {\sl learnability estimation}:  for monotone target functions, the error of the decision tree hypothesis constructed by these heuristics can be estimated with {\sl polylogarithmically} many labeled examples, exponentially smaller than the number necessary to run these heuristics, and indeed, exponentially smaller than information-theoretic minimum required to learn a good decision tree.   This adds to a small but growing list of fundamental learning algorithms that have been shown to be amenable to learnability estimation. 

En route to this result, we design and analyze sample-efficient {\sl minibatch} versions of top-down decision tree learning heuristics and show that they achieve the same provable guarantees as the full-batch versions.  We further give ``active local'' versions of these heuristics: given a test point $x^\star$, we show how the label $T(x^\star)$ of the decision tree hypothesis $T$ can be computed with polylogarithmically many labeled examples, exponentially smaller than the number necessary to learn~$T$. 
\end{abstract} 

\thispagestyle{empty}

\newpage 
\setcounter{page}{1}

\section{Introduction} 

We study the problem of {\sl estimating learnability}, recently introduced by Kong and Valiant~\cite{KV18} and Blum and Hu~\cite{BH18}.  Consider a learning algorithm $\mathcal{A}$ and a dataset~$S$ of {\sl unlabeled} examples.  Can we estimate the performance of $\mathcal{A}$ on $S$---that is, the error of the hypothesis that $\mathcal{A}$ would return if we were to label the entire dataset $S$ and train $\mathcal{A}$ on it---by labeling only very few of the examples in $S$?  Are there learning tasks and algorithms for which an accurate estimate of learnability can be obtained with far fewer labeled examples than the information-theoretic minimum required to learn a good hypothesis?  

\paragraph{Motivating applications.}  Across domains and applications, the labeling of datasets is often an expensive process, requiring either significant computational resources or a large number of person-hours.  There are therefore numerous natural scenarios in which an efficient learnability estimation procedure could serve as a useful exploratory precursor to learning.  For example, suppose the error estimate returned by this procedure is large.  This tells us that if we were to label the entire dataset~$S$ and run $\mathcal{A}$ on it, the error of the hypothesis $h$ that $\mathcal{A}$ would return is large.  With this information, we may decide that $h$ would not have been of much utility anyway, thereby saving ourselves the resources and effort to label the entire dataset $S$ (and to run $\mathcal{A}$).  Alternatively, we may decide to collect more data or to enlarge the feature space of $S$, in hopes of improving the performance of~$\mathcal{A}$.  The learnability estimation procedure could again serve as a guide in this process, telling us {\sl how much} the performance of $\mathcal{A}$ would improve with these decisions.  Relatedly, such a procedure could be useful for hyperparameter tuning, where the learning algorithm $\mathcal{A}$ takes as input a parameter~$\rho$, and its performance improves with $\rho$, but its time and sample complexity also increases with $\rho$.  The learnability estimation procedure enables us to efficiently determine the best choice of $\rho$ for our application at hand, and run $\mathcal{A}$ just a single time with this value of $\rho$.  As a final example, such a procedure could also be useful for dataset selection: given unlabeled training sets $S_1,\ldots,S_m$, and access to labeled examples from a test distribution $\mathcal{D}$, we can efficiently determine the $S_i$ for which~$\mathcal{A}$ would produce a hypothesis that achieves the smallest error with respect to $\mathcal{D}$. 

\paragraph{Prior works on estimating learnability.} While this notion is still relatively new, there are already a number of works studying it in a variety of settings, including robust linear regression~\cite{KV18}, learning unions of intervals and $k$-Nearest-Neighbor algorithms~\cite{BH18}, contextual bandits~\cite{KVB20}, learning Lipschitz functions, and  the Nadaraya–Watson estimator in kernel regression~\cite{BBG20}.

A striking conceptual message has emerged from this line of work: it is often possible to estimate learnability with far fewer labeled examples than the number required to run the corresponding algorithm, and indeed, far fewer than the information-theoretic minimum required to learn a good hypothesis.

\subsection{Top-down decision tree learning}
\label{sec:background}
 We study the problem of estimating learnability in the context of {\sl decision tree learning}.  Specifically, we focus on  {\sl top-down} decision tree learning heuristics such as ID3, C4.5, and CART.  These classic and simple heuristics continue to be widely employed in everyday machine learning applications and enjoy significant empirical success.  They are also the core subroutine in modern, state-of-the-art ensemble methods such as random forests and gradient boosted trees. 

We briefly describe how these top-down heuristics work, deferring the formal description to the main body of this paper.  Each such heuristic $\TopDown_{\mathscr{G}}$ is defined by {\sl impurity function} $\mathscr{G} : [0,1]\to [0,1]$ which determines its splitting criterion.\footnote{Impurity functions $\mathscr{G}$ are restricted to be concave, symmetric around $\frac1{2}$, and to satisfy $\mathscr{G}(0) = \mathscr{G}(1) = 0$ and $\mathscr{G}(\frac1{2}) = 1$.  For example, ID3 and C4.5 use the binary entropy function $\mathscr{G}(p) = \mathrm{H}(p)$, and the associated purity gain is commonly referred to as information gain; CART uses the Gini criterion $\mathscr{G}(p) = 4p(1-p)$; Kearns and Mansour proposed and analyzed the function $\mathscr{G}(p) = 2\sqrt{p(1-p)}$~\cite{KM99}. The work of Dietterich, Kearns, and Mansour~\cite{DKM96} provides a detailed discussion and experimental comparison of various impurity functions.}  $\TopDown_{\mathscr{G}}$ takes as input a labeled dataset $S \sse \mathcal{X} \times \zo$ and a size parameter $t\in \N$, and constructs a size-$t$ decision tree for $S$ in a {\sl greedy}, {\sl top-down} fashion.  It begins by querying $\Ind[x_i \ge \theta]$ at the root of the tree, where~$x_i$ and $\theta$ are chosen to maximize the {\sl purity gain with respect to~$\mathscr{G}$}: 
\[ \mathscr{G}(\E[\by]) -  \big(\Pr[\bx_i \ge \theta] \cdot \mathscr{G}(\E[\,\by \mid \bx_i \ge \theta\,]) + \Pr[\bx_i < \theta] \cdot \mathscr{G}(\E[\,\by \mid \bx_i < \theta\,]) \big),\]
where the expectations and probabilities are with respect to $(\bx,\by) \sim S$.  
More generally, $\TopDown_{\mathscr{G}}$ grows its current tree $T$ by splitting a leaf $\ell \in T^\circ$ with a query to $\Ind[x_i\ge \theta]$, where $\ell$, $x_i$, and $\theta$ are chosen to maximize: 
\[ 
 \mathrm{PurityGain}_{\mathscr{G},S}(\ell,i,\theta) \coloneqq \Pr[\text{$\bx$ reaches $\ell$}] \cdot \mathrm{LocalGain}_{\mathscr{G},S}(\ell,i,\theta), \] 
 where
\begin{align*}
  \mathrm{LocalGain}_{\mathscr{G},S}(\ell,i,\theta) &\coloneqq \mathscr{G}(\E[\,\by \mid \text{$\bx$ reaches $\ell$}\,])  \\ 
  &\quad -   \big(\Pr[\bx_i\ge \theta]\cdot \mathscr{G}(\E[\,\by\mid \text{$\bx$ reaches $\ell$, $\bx_i \ge \theta$}\,]) \\
  &\quad\ \,\,  + \Pr[\bx_i < \theta] \cdot  \mathscr{G}(\E[\,\by\mid \text{$\bx$ reaches $\ell$, $\bx_i < \theta$}\,])\big).
\end{align*}

\paragraph{Provable guarantees for monotone target functions.}  Motivated by the popularity and empirical success of these top-down heuristics, there has been significant interest and efforts in establishing provable guarantees on their performance~\cite{Kea96,DKM96,KM99,FP04,Lee09,BDM19,BDM19-smooth,BLT20-ITCS,BLT3}.  The starting point of our work is a recent result of Blanc, Lange, and Tan~\cite{BLT3}, which provides a guarantee on their performance when run on {\sl monotone} target functions, with respect to the uniform distribution:
\bigskip

\noindent {\bf Theorem 2 of~\cite{BLT3}.} {\it Let $f : \bits^d \to \zo$ be a monotone target function and $\mathscr{G}$ be any impurity function.  For $s\in \N$ and $\eps,\delta \in (0,\frac1{2})$, let $t = s^{\Theta(\log s)/\eps^2}$ and $\bS$ be a set of $n$ labeled training examples $(\bx,f(\bx))$ where $\bx \sim \bits^d$ is uniform random, and $n = \tilde{O}(t)\cdot \poly(\log d,\log(1/\delta)).$  With probability at least $1-\delta$ over the randomness of $\bS$, the size-$t$ decision tree hypothesis constructed by $\TopDown_{\mathscr{G}}(t,\bS)$ satisfies $\error_f(T) \coloneqq \Pr_{\bx\sim\bits^d}[T(\bx)\ne f(\bx)] \le \opt_s + \eps$, where $\opt_s$ denotes the error of the best size-$s$ decision tree for $f$.
} 
\bigskip  


We refer the reader to the introduction of~\cite{BLT3} for a discussion of why assumptions on the target function are necessary in order to establish provable guarantees. Briefly, as had been noted by Kearns~\cite{Kea96}, there are examples of simple non-monotone target functions $f: \bits^d \to \zo$, computable by decision trees of constant size, for which any impurity-based heuristic may build a complete tree of size $\Omega(2^d)$ before achieving any non-trivial accuracy. Monotonicity is a natural way of excluding these adversarial functions, and for this reason it is one of the most common assumptions in learning theory. Results for monotone functions tend to be good proxies for the performance of learning algorithms on real-world datasets, which also do not exhibit these adversarial structures.

\paragraph{Our contributions.}  We give strengthened provable guarantees on the performance of top-down decision tree learning heuristics, focusing on sample complexity.  Our three main contributions are as follows:

\begin{enumerate}[leftmargin = 15pt]
\item {\sl Minibatch top-down decision tree learning.} We introduce and analyze $\textsc{MiniBatchTopDown}_\mathscr{G}$, a {\sl minibatch} version of $\TopDown_{\mathscr{G}}$ where the purity gain associated with each split is estimated with only polylogarithmically many samples within the dataset $S$ rather than all of $S$.  For all impurity functions $\mathscr{G}$, we show that $\textsc{MiniBatchTopDown}_\mathscr{G}$  achieves the same provable guarantees that those that~\cite{BLT3} had established for the full-batch version $\TopDown_\mathscr{G}$.

\item {\sl Active local learning.}  We then study $\textsc{MiniBatchTopDown}_{\mathscr{G}}$ within the recently-introduced {\sl active local learning} framework of Backurs, Blum, and Gupta~\cite{BBG20}, and show that it admits an efficient active local learner.  Given active access to an {\sl unlabeled} dataset $S$ and a test point~$x^\star$, we show how $T(x^\star)$ can be computed by labeling only polylogarithmically many of the examples in $S$, where $T$ is the decision tree hypothesis that $\textsc{MiniBatchTopDown}_{\mathscr{G}}$ would construct if we were to label {\sl all} of $S$ and train $\textsc{MiniBatchTopDown}_{\mathscr{G}}$ on it. 

\item {\sl Estimating learnability.} Building on  our results above, we show that $\textsc{MiniBatchTopDown}_{\mathscr{G}}$ is amendable to highly-efficient learnability estimation.  Given active access to an unlabeled dataset $S$, we show that the error of $T$ with respect to any test distribution can be approximated by labeling only polylogarithmically many of the examples in~$S$, where $T$ is the decision tree hypothesis that $\textsc{MiniBatchTopDown}_{\mathscr{G}}$ would construct if we were to label all of $S$ and train $\textsc{MiniBatchTopDown}_{\mathscr{G}}$ on it. 
\end{enumerate} 

\subsection{Formal statements of our results}

\paragraph{Feature space and distributional assumptions.}  We work in the setting of binary attributes and binary classification, i.e.~we focus on the task of learning a target function $f : \bits^d \to \zo$.  We will assume the learning algorithm receives uniform random examples $\bx \sim \bits^d$, either labeled or unlabeled.  The error of a decision tree hypothesis $T : \bits^d \to \zo$ with respect to $f$ is defined to be $\error_f(T) \coloneqq \Pr[f(\bx) \ne T(\bx)]$ where $\bx\sim \bits^d$ is uniform random.  We write $\opt_s(f)$ to denote $\min\{ \error_f(T) \colon \text{$T$ is a size-$s$ decision tree} \}$; when $f$ is clear from context we simply write $\opt_s$.  We will also be interested in the error of $T$ with respect to general test sets $(\Pr_{(\bx,\by)\sim S_{\mathrm{test}}}[T(\bx)\ne \by]$) and general test distributions $(\Pr_{(\bx,\by)\sim \mathcal{D}_{\mathrm{test}}}[T(\bx)\ne \by])$.

\paragraph{Notation and terminology.} For any decision tree $T$, we say the size of $T$ is the number of leaves in~$T$. We refer to a decision tree with unlabeled leaves as a {\sl partial tree}, and write $T^\circ$ to denote such trees.  
For a leaf $\ell$ of a partial tree $T^\circ$, we write $|\ell|$ to denote its depth within~$T^\circ$, the number of attributes queried along the path that leads to $\ell$.  We say that an input $x\in \bits^d$ is {\sl consistent with a leaf $\ell$} if $x$ reaches $\ell$ within $T^\circ$, and we write $\ell_{T^\circ}(x)$ to denote the (unique) leaf $\ell$ of $T^\circ$ that $x$ is consistent with. 
A function $f: \bits^d \to \{0,1\}$ is said to be {\sl monotone} if for every coordinate $i \in [d]$, it is either non-decreasing with respect to $i$ (i.e.~$f(x) \leq f(y)$ for all $x,y \in \bits^d$ such that $x_i \leq y_i$) or non-increasing with respect to $i$ (i.e.~$f(x) \geq f(y)$ for all $x,y\in\bits^d$ such that $x_i \leq y_i$).  We use {\bf boldface} to denote random variables (e.g.~$\bx\sim\bits^d$), and unless otherwise stated, all probabilities and expectations are with respect to the uniform distribution.  For $p\in [0,1]$, we write $\mathrm{round}(p)$ to denote $\Ind[p\ge \frac1{2}]$.  We reserve $S$ to denote a labeled dataset and $S^\circ$ to denote an unlabeled dataset.
\bigskip

We are now ready to describe our algorithms and state our main results.

\begin{definition}[Minibatch]
Let $S$ be a labeled dataset.  A \emph{minibatch from $S$}, denoted $\bB \sim \Batch_b(S)$, is a set of $b$ uniform random points $(x,y)$ chosen without replacement from $S$.  More generally, for a leaf $\ell$, a \emph{minibatch consistent with $\ell$ from $S$}, denoted $\bB\sim \Batch_b(S, \ell)$, is a set of $b$ uniformly random pairs chosen without replacement from among $(x,y) \in S$ such that $x$ is consistent with $\ell$.  (In both cases, if there are fewer than $b$ such points, we return all of them.) Minibatches from unlabeled datasets $S^\circ$ are defined analogously. \end{definition}

\begin{definition}[Minibatch completion of partial trees]Given a partial tree $T^\circ$, we write $T^{\circ}_{\Batch_b(S)}$ to denote the tree obtained by labeling each leaf $\ell \in T^\circ$ with $\mathrm{round}(\E_{(\bx,f(\bx))\sim \bB}[f(\bx)])$ where $\bB \sim \Batch_b(S,\ell)$. 
\end{definition} 
\begin{figure}[H]
  \captionsetup{width=.9\linewidth}
\begin{tcolorbox}[colback = white,arc=1mm, boxrule=0.25mm,left=20pt]
\vspace{3pt} 

\hspace{-10pt}$\MiniBatchTopDown_{\mathscr{G}}$($t,b,S$):  \vspace{6pt}

 Initialize $T^\circ$ to be the empty tree. \vspace{4pt} 
 
  Define $D \coloneqq \log t + \log\log t$.\vspace{4pt}

 while ($\size(T^\circ) < t$) \{  \\

\vspace{-10pt}
\begin{enumerate}[leftmargin=20pt]
\item \hspace{-5pt} {\sl Score:}  For each leaf $\ell \in T^\circ$ of depth at most $D$,  \violet{draw $\bB \sim \Batch_b(S,\ell)$.}  For each coordinate $i\in [d]$,  compute:
\begin{align*}  
\violet{\mathrm{PurityGain}_{\mathscr{G},\bB}}(\ell,i) &\coloneqq 2^{-|\ell|} \cdot \mathrm{LocalGain}_{\mathscr{G},\bB} (\ell,i), \text{ where}\\
  \violet{\mathrm{LocalGain}_{\mathscr{G},\bB}}(\ell,i) &\coloneqq \mathscr{G}(\E[\violet{f(\bx)}])  \\ 
  &\quad -   \big(\lfrac1{2} \, \mathscr{G}(\E[\,\violet{f(\bx)}\mid \text{$\bx_i =-1$}\,]) + \lfrac1{2} \, \mathscr{G}(\E[\,\violet{f(\bx)}\mid \text{$\bx_i = 1$}\,])\big),
\end{align*}
where the expectations are with respect to \violet{$(\bx,\violet{f(\bx)})\sim \bB$}.
\item {\sl Split:} Let $(\ell^\star,i^\star)$ be the tuple that maximizes $\violet{\mathrm{PurityGain}_{\mathscr{G},\bB}}(\ell,i)$.  Grow $T^\circ$ by splitting $\ell^\star$ with a query to $x_{i^\star}$.\\
\end{enumerate}
\vspace{-18pt}

\ \ \}  \vspace{4pt}

Output $T^\circ_{\Batch_b(S)}$.
\vspace{3pt}
\end{tcolorbox}
\label{fig:MiniBatchTopDown}
\caption{$\MiniBatchTopDown_{\mathscr{G}}$ takes as input   a size parameter $t$, a minibatch size $b$, and a labeled dataset $S$.  It outputs a size-$t$ decision tree hypothesis for $f$.}
\end{figure}

\blue{$\MiniBatchTopDown_{\mathscr{G}}$ is a minibatch version of $\TopDown_{\mathscr{G}}$, which we described informally in~\Cref{sec:background} and include its full pseudocode in~\Cref{sec:upper bound mini batch overview}.
$\MiniBatchTopDown_{\mathscr{G}}$ is more efficient than $\TopDown_{\mathscr{G}}$ in two respects:  first, purity gains and completions are computed with respect to a minibatch $\bB$ of size $b$ instead of all the entire dataset $S$; second, $\MiniBatchTopDown_{\mathscr{G}}$ never splits a leaf of depth greater than $D$, and hence constructs a decision tree of small size {\sl and} small depth, rather than just small size.  (Looking ahead, both optimizations will be crucial for the design of our sample-efficient active local learning and learnability estimation procedures.)  

Our first result shows that $\MiniBatchTopDown_{\mathscr{G}}$ achieves the same performance guarantees as those that~\cite{BLT3} had established for the full-batch version $\TopDown_{\mathscr{G}}$: }

\begin{theorem}[Provable guarantees for \MiniBatchTopDown; informal version]
    \label{thm:upper bound mini batch}
    Let $f: \bits^d \to \zo$ be a monotone target function and fix an impurity function $\mathscr{G}$. For any $s \in \mathbb{N}$, $\eps, \delta \in (0,\frac{1}{2})$, let $t = s^{\Theta(\log s)/\eps^2}$, and $\bS$ be a set of $n$ labeled training examples $(\bx, f(\bx))$ where $\bx \sim \bits^d$ is uniform random, and  
    \[ n = \tilde{O}(t) \cdot \poly(\log d, \log(1 / \delta)).\]    If the minibatch size is at least
    \[ b = \polylog(t) \cdot \poly( \log d,\log(1/\delta)),\] 
  then with probability at least $1 - \delta$ over the randomness of $\bS$ and the draws of minibatches from within~$\bS$, the size-$t$ decision tree hypothesis constructed by $\MiniBatchTopDown_{\mathscr{G}}(t, b, \bS)$ satisfies $\error_f(T) \le \opt_s + \eps$.
\end{theorem}

\Cref{thm:upper bound mini batch} shows that it suffices for the minibatch size $b$ of $\MiniBatchTopDown_{\mathscr{G}}$ to depend polylogarithmically on $t$; in contrast, the full-batch version $\TopDown_{\mathscr{G}}$ uses the entire set $S$ to compute purity gains and determine its splits, and $|S|=n$ has a superlinear dependence on $t$.

Our next algorithm is an implementation of $\MiniBatchTopDown_{\mathscr{G}}$ within the {\sl active local learning} framework of Backurs, Blum, and Gupta~\cite{BBG20}; see~\Cref{fig:LocalAlgorithm}.

%

\begin{figure}[h!tb]
  \captionsetup{width=.9\linewidth}
\begin{tcolorbox}[colback = white,arc=1mm, boxrule=0.25mm,left=20pt]
\vspace{3pt} 

\hspace{-10pt}$\LocalLearner_{\mathscr{G}}$($t,b, \violet{S^\circ}, x^\star$):  \vspace{6pt}

Initialize $T^\circ$ to be the empty tree. \vspace{4pt}

Define $D \coloneqq \log t + \log\log t$. \vspace{4pt}

Initialize $e \coloneqq 1$ and let $\bB^\circ_{\mathrm{strands}}$ be $b$ uniform random points from $\bits^d$. \vspace{4pt}

 while ($e < t$) \{  \\

\vspace{-10pt}
\begin{enumerate}[leftmargin=20pt]
\item \hspace{-5pt} {\sl Score:} For each leaf $\ell \in \{\ell_{T^\circ}(x)\colon x \in \bB^\circ_{\mathrm{strands}}\cup \{x^\star\}\}$ of depth at most $D$, \violet{draw $\bB^\circ\sim \Batch_b(S^\circ,\ell),$ query $f$'s values on these points.}  For each coordinate $i\in [d]$, compute:
\begin{align*}  \violet{\mathrm{PurityGain}_{\mathscr{G},\bB^\circ}}(\ell,i) &\coloneqq 2^{-|\ell|} \cdot \mathrm{LocalGain}_{\mathscr{G},\bB^\circ} (\ell,i), \text{ where}\\
 \violet{\mathrm{LocalGain}_{\mathscr{G},\bB^\circ}}(\ell,i) &\coloneqq \mathscr{G}(\E[\violet{f(\bx)}])  \\ 
  &\quad -   \big(\lfrac1{2} \, \mathscr{G}(\E[\,\violet{f(\bx)}\mid \text{$\bx_i =-1$}\,]) + \lfrac1{2} \, \mathscr{G}(\E[\,\violet{f(\bx)}\mid \text{$\bx_i = 1$}\,])\big),
\end{align*}
where the expectations are with respect to \violet{$\bx \sim \bB^\circ$}. 
\item {\sl Split:} Let $(\ell^\star,i^\star)$ be the tuple that maximizes $\violet{\mathrm{PurityGain}_{\mathscr{G},\bB^\circ}}(\ell,i)$.  Grow $T^\circ$ by splitting $\ell^\star$ with a query to $x_{i^\star}$.
\item {\sl Estimate size:} Update our size estimate to
\begin{align*}
    \boldsymbol{e} = \Ex_{\bx \sim \violet{\bB^\circ_{\mathrm{strands}}}}\big[2^{|\ell_{T^{\circ}}(\bx)|}\big].
\end{align*}
\end{enumerate}
\vspace{-18pt}

\ \ \}  \vspace{4pt} 

\violet{Draw $\bB^\circ\sim \Batch_b(S^\circ,\ell_{T^\circ}(x^\star))$ and query $f$'s values on these points.} \vspace{4pt} \\ 
\violet{Output $\mathrm{round}(\Ex_{\bx\sim\bB^\circ}[f(\bx)])$.}

\vspace{3pt}
\end{tcolorbox}
\caption{$\LocalLearner_{\mathscr{G}}$ takes as input a size parameter $t$, a minibatch size $b$, an {\sl unlabeled} dataset $S^\circ$, and an input $x^\star$.  It selectively queries $f$'s values on a few points within $S^\circ$ and outputs $T(x^\star)$, where $T$ is a tree of size approximately $t$ that $\MiniBatchTopDown_{\mathscr{G}}$ would return if  we were to label all of $S^\circ$ and train $\MiniBatchTopDown_{\mathscr{G}}$ on it.}
\label{fig:LocalAlgorithm}
\end{figure}

\begin{theorem}[Active local version of $\MiniBatchTopDown$; informal version]
    \label{thm:local learner}
    Let $f: \bits^d \to \zo$ be a target function, $\mathscr{G}$ be an impurity function, and  $\violet{S^\circ}$ be an {\sl unlabeled} training set.     For all $t \in \mathbb{N}$, $\violet{\eta}, \delta \in (0,\frac{1}{2})$, if the minibatch size  is at least $b = \poly(\log t, \log d,1/\eta,\log(1/\delta))$,
  then with probability at least $1 - \delta$ over the randomness of $\bB^\circ_{\mathrm{strands}}$, we have that for all $x^\star \in \bits^d$, $\LocalLearner_{\mathscr{G}}(t,b,S^\circ,x^\star)$ labels 
  \[ q = O(b^2\log t) = \polylog(t) \cdot \poly(\log d,1/\eta,\log(1/\delta)) \]  points within $S^\circ$ and returns $T(x^\star)$, where $T$ is the size-$t'$ decision tree hypothesis that \[ \MiniBatchTopDown_{\mathscr{G}}(t',b,S)\]  would construct, $t' \in t(1\pm \violet{\eta})$, and $S$ is the labeled dataset obtained by labeling all of $S^\circ$ with $f$'s values.\footnote{To ensure that $\LocalLearner_\mathscr{G}$ consistently labels all $x^\star$ according to the {\sl same} tree $T$, we run all invocations of $\LocalLearner_{\mathscr{G}}$  with the same outcomes of randomness for $\bB_{\mathrm{strands}}^\circ$ and  draws of minibatches. Similarly, if one then wished to actually construct this tree $T$, they would run $\MiniBatchTopDown_{\mathscr{G}}$ with these same outcomes of randomness.\label{footnote:assumption}}
\end{theorem}


\Cref{thm:local learner} yields, as a fairly straightforward consequence, our learnability estimation procedure $\mathrm{Est}_{\mathscr{G}}$ that estimates the performance of $\MiniBatchTopDown_{\mathscr{G}}$ with respect to any test set $S_{\mathrm{test}}$: 

\begin{theorem}[Estimating learnability of $\MiniBatchTopDown$; informal version] 
\label{thm:estimate-learn}
   Let $f: \bits^d \to \zo$ be a target function, $\mathscr{G}$ be an impurity function,  $\violet{S^\circ}$ be an {\sl unlabeled} training set, and $S_{\mathrm{test}}$ be a labeled test set. For all $t \in \mathbb{N}$ and $\violet{\eta}, \delta \in (0,\frac{1}{2})$, if the minibatch size $b$ is as in~\Cref{thm:local learner}, then         
    with probability at least $1 - \delta$ over the randomness of the draws of minibatches from within \violet{$S^\circ$}, $\Est_{\mathscr{G}}(t,b,S^\circ,S_{\mathrm{test}})$ labels 
    \[ q = O(|S_{\mathrm{test}}| \cdot  b \log t + b^2 \log t) = |S_{\mathrm{test}}| \cdot \polylog(t) \cdot \poly(\log d,1/\eta,\log(1/\delta)) \]  
    points within $S^\circ$ and returns the error of $T$ with respect to $S_{\mathrm{test}}$,
    \[ \error_{S_{\mathrm{test}}}(T) \coloneqq \Prx_{(\bx,\by)\sim S_{\mathrm{test}}}[T(\bx) \ne \by], \] 
where $T$ is as in~\Cref{thm:local learner}. 
\end{theorem} 

We remark that~\Cref{thm:upper bound mini batch} requires the training set be composed of independent draws of $(\bx, f(\bx))$ where $\bx$ is drawn uniformly from $\bits^d$. On the other hand, in~\Cref{thm:local learner,thm:estimate-learn}, the high probability guarantees hold for any fixed choice of training set $S^\circ$. Similarly, in \Cref{thm:estimate-learn}, $S_{\mathrm{test}}$ can be arbitrarily chosen. Indeed, as an example application of~\Cref{thm:estimate-learn}, we can let $\bS_{\mathrm{test}}$ be $\Theta(\log(1/\delta)/\eps^2)$ many labeled examples $(\bx,\by)$ drawn from an arbitrary test distribution $\mathcal{D}_{\mathrm{test}}$ over $\bits^d \times \zo$, where the marginal over $\bits^d$ need not be uniform and the the labels need not be consistent with $f$. With probability at least $1-\delta$, the output of $\mathrm{Est}_{\mathscr{G}}$ will be within $\pm \eps$ of $\Pr_{(\bx,\by)\sim\mathcal{D}_{\mathrm{test}}}[T(\bx)\ne \by]$.

\section{Proof overview of \Cref{thm:upper bound mini batch}}
\label{sec:upper bound mini batch overview}

Our proof of~\Cref{thm:upper bound mini batch} builds upon and extends the analysis in~\cite{BLT3}. (Recall that~\cite{BLT3} analyzed the full-batch version $\TopDown_{\mathscr{G}}$, which we have included in the figure below, and their guarantee concerning its performance is their  Theorem 2, which we have stated in~\Cref{sec:background} of this paper).  In this section we give a high-level overview of both \cite{BLT3}'s and our proof strategy, in tandem with a description of the technical challenges that arise as we try to strengthen~\cite{BLT3}'s Theorem 2 to our~\Cref{thm:upper bound mini batch}. 

\begin{figure}[H]
  \captionsetup{width=.9\linewidth}
\begin{tcolorbox}[colback = white,arc=1mm, boxrule=0.25mm,left=20pt]
\vspace{3pt} 

\hspace{-10pt}$\TopDown_{\mathscr{G}}$($t,S$):  \vspace{6pt}

 Initialize $T^\circ$ to be the empty tree. \vspace{4pt} 

 while ($\size(T^\circ) < t$) \{  \\

\vspace{-10pt}
\begin{enumerate}
\item \hspace{-5pt} {\sl Score:} For each leaf $\ell \in T^\circ$ and coordinate $i\in [d]$,  compute:
\begin{align*}  \puritygainS(\ell,i) &\coloneqq 2^{-|\ell|} \cdot \localgainS (\ell,i), \text{ where}\\
  \localgainS(\ell,i) &\coloneqq \mathscr{G}(\E[\,\violet{f(\bx)} \mid \text{$\bx$ reaches $\ell$}\,])  \\ 
  &\quad -   \big(\lfrac1{2} \cdot \mathscr{G}(\E[\,\violet{f(\bx)}\mid \text{$\bx$ reaches $\ell$, $\bx_i =-1$}\,]) \\
  &\quad\ \,  + \lfrac1{2} \cdot  \mathscr{G}(\E[\,\violet{f(\bx)}\mid \text{$\bx$ reaches $\ell$, $\bx_i = 1$}\,])\big),
\end{align*}
where $(\bx,\violet{f(\bx)})\sim S$.
\item {\sl Split:} Let $(\ell^\star,i^\star)$ be the tuple that maximizes $\puritygainS(\ell,i)$.  Grow $T^\circ$ by splitting $\ell^\star$ with a query to $x_{i^\star}$.\\
\end{enumerate}
\vspace{-18pt}

\ \ \}  \vspace{4pt}

Output $T^{\circ}_S$, the completion of $T^\circ$ with respect to $S$: label each leaf $\ell \in T^\circ$ with $\mathrm{round}(\E[\,f(\bx)\mid\text{$\bx$ reaches $\ell$}\,])$, where $(\bx,f(\bx))\sim S$.

\vspace{3pt}
\end{tcolorbox}
\end{figure}

Let $f:\bits^d \to \zo$ be a monotone function and fix an impurity function $\mathscr{G}$. Let $T^\circ$ be a partial tree that is being built by either $\TopDown_{\mathscr{G}}$ or $\MiniBatchTopDown_{\mathscr{G}}$.  Recall that  $\TopDown_{\mathscr{G}}$ and $\MiniBatchTopDown_{\mathscr{G}}$ compute, for each leaf $\ell \in T^\circ$ and coordinate $i\in [d]$, $\mathrm{PurityGain}_{\mathscr{G},S}(\ell,i)$ and $\mathrm{PurityGain}_{\mathscr{G},\bB}(\ell,i)$ respectively.  Both these quantities can be thought of as estimates of the {\sl true} purity gain: 
\begin{align*} 
\mathrm{PurityGain}_{\mathscr{G},f}(\ell,i) &\coloneqq 2^{-|\ell|} \cdot \mathrm{LocalGain}_{\mathscr{G},f}(\ell,i)  \text{ where}\\
  \mathrm{LocalGain}_{\mathscr{G},f}(\ell,i) &\coloneqq \mathscr{G}(\E[\,f(\bx)\mid \text{$\bx$ reaches $\ell$}\,])  \\ 
  &\quad -   \big(\lfrac1{2} \, \mathscr{G}(\E[\,f(\bx)\mid \text{$\bx$ reaches $\ell$, $\bx_i =-1$}\,]) \\
  &\quad \ \,+ \lfrac1{2} \, \mathscr{G}(\E[\,f(\bx)\mid \text{$\bx$ reaches $\ell$, $\bx_i = 1$}\,])\big),
\end{align*} 
where here and throughout this section, all expectations are with respect to a uniform random $\bx \sim \bits^d$.  The fact that $\MiniBatchTopDown_{\mathscr{G}}$'s estimates of this true purity gain are based on minibatches $\bB$ of size exponentially smaller than that of the full sample set $S$---and hence could be exponentially less accurate---is a major source of technical challenges that arise in extending~\cite{BLT3}'s guarantees for $\TopDown_{\mathscr{G}}$ to $\MiniBatchTopDown_{\mathscr{G}}$.

~\cite{BLT3} considers the potential function:
\begin{equation*}
    \Gimpurity_{f}(T^\circ) \coloneqq \sum_{\text{leaves $\ell \in T^\circ$}} 2^{-|\ell|} \cdot \mathscr{G}(\E[f_\ell]).  
\end{equation*}
The following fact about this potential function $\Gimpurity_{f}$ is straightforward to verify (and is proved in~\cite{BLT3}):
\begin{fact} 
\label{fact:useful-properties} For any partial tree $T^\circ$, leaf $\ell \in T^\circ$, and coordinate $i \in [d]$, let $\tilde{T}^\circ$ be the tree obtained from $T^\circ$ by splitting $\ell$ with a query to $x_i$. Then,
\[             \Gimpurity_{f}(\tilde{T}^\circ) = \Gimpurity_{f}(T^\circ) - \puritygainf(\ell, i).
   \] 
   \end{fact}

A key ingredient in \cite{BLT3}'s analysis is a proof that as long as $\error_f(T^\circ_S) > \opt_s + \eps$ (where $T^\circ_S$ denotes the completion of $T^\circ$ with respect to the full batch $S$; see~Section 5), there must be a leaf $\ell \in T^\circ$ and coordinate $i$ with high true purity gain, $\puritygainf(\ell, i) \ge \poly(\eps/t)$.   Since $\TopDown_{\mathscr{G}}$'s estimates $\mathrm{PurityGain}_{\mathscr{G},S}$ of $\puritygainf$ are with respect to a sample of size $|S| \ge \poly(t/\eps)$, it follows that $\TopDown_{\mathscr{G}}$ will make a split for which the true purity gain is indeed $\poly(\eps/t)$.  By~\Cref{fact:useful-properties}, such a split constitutes good progress with respect to the potential function  $\Gimpurity_f$.  Summarizing,~\cite{BLT3} that shows until $\error_f(T^\circ_S) < \opt_s + \eps$ is achieved, {\sl every} split that $\TopDown_{\mathscr{G}}$ makes has high true purity gain, and hence constitutes good progress with respect to the potential function $\Gimpurity_f$.


The key technical difficulty in analyzing $\MiniBatchTopDown_{\mathscr{G}}$ instead of $\TopDown_{\mathscr{G}}$ is that $\MiniBatchTopDown_{\mathscr{G}}$ is not guaranteed to choose a split with high true purity gain: it could make splits for which its estimate $\mathrm{PurityGain}_{\mathscr{G},\bB}(\ell,i)$ is high, but the true purity gain $\mathrm{PurityGain}_{\mathscr{G},f}(\ell,i)$ is actually tiny.  In fact, unless we use batches of size $b \ge \poly(t)$, exponentially larger than the $b = \polylog(t)$ of \Cref{thm:upper bound mini batch}, $\MiniBatchTopDown_{\mathscr{G}}$ could make splits that result in zero true purity gain, and hence constitute zero progress with respect to the potential function $\Gimpurity_f$.

To overcome this challenge, we instead show that {\sl most} splits $\MiniBatchTopDown_{\mathscr{G}}$ makes have high true purity gain. We first show that with high probability over the draws of minibatches $\bB$, if $\MiniBatchTopDown_{\mathscr{G}}$ splits a leaf that is neither too shallow nor too deep within $T^\circ$, then this split has high true purity gain (\Cref{lem:medium splits good}). We then show the following two lemmas:
\begin{enumerate}[leftmargin=20pt]
    \item \Cref{lem:max depth}: If $\MiniBatchTopDown_{\mathscr{G}}$ splits a leaf of $T^\circ$ that is sufficiently deep, then it must be the case that $\error_f(T^{\circ}_{\Batch_b(S)}) \le \opt_s + \eps$, i.e.~the current tree already achieves sufficiently small error.  With this Lemma, we are able to define $\MiniBatchTopDown_{\mathscr{G}}$ to never split a leaf that is too deep, while retaining guarantees on its performance.  
    \item \Cref{lemma:few splits shallow}: Only a small fraction of splits made by $\MiniBatchTopDown_{\mathscr{G}}$ can be too shallow. 
\end{enumerate}
Combining the above lemmas, we are able to prove \Cref{thm:upper bound mini batch}. 
\section{Proof of~\Cref{thm:upper bound mini batch}}
\label{section:minbatchworks}



We first need a couple of definitions:

\begin{definition}[H\"{o}lder continuous]
    For $C, \alpha > 0$, an impurity function $\mathscr{G}: [0,1] \to [0,1]$ is {\sl ($C, \alpha$)-H\"{o}lder continuous} if, for all $a,b \in [0,1]$,
    \begin{align*}
        |\mathscr{G}(a) - \mathscr{G}(b)| \leq C|a - b|^\alpha.
    \end{align*}
\end{definition}

\begin{definition}[Strong concavity]
    For $\kappa > 0$, an impurity function $\mathscr{G}: [0,1] \to [0,1]$ is {\sl $\kappa$-strongly concave} if for all $a,b \in [0,1]$, 
    \[ \frac{\mathscr{G}(a)+\mathscr{G}(b)}{2}  \le \mathscr{G}\left(\frac{a+b}{2}\right) - \frac{\kappa}{2}\cdot (b-a)^2. \] 
\end{definition}
\begin{theorem}[Provable guarantee for \MiniBatchTopDown; formal version of \Cref{thm:upper bound mini batch}]
    \label{thm:upper bound mini batch formal}
    Let $f: \bits^d \to \zo$ be a monotone target function and $\mathscr{G}$ be any $\kappa$-strongly concave and $(C, \alpha)$-H\"{o}lder continuous impurity function. For any $s \in \mathbb{N}$, $\eps, \delta \in (0,\frac{1}{2})$, let $t = s^{\Theta(\log(s))/\eps^2}$, and $\bS$ be a set of $n$ labeled training examples $(\bx, f(\bx))$ where $\bx \sim \bits^d$ is uniform random, and
    \begin{align*}
        n = t \cdot \Omega\left(\left(\frac{C^2 \log(s)^4} {\kappa^2  \eps^4} \right)^{\frac{1}{\alpha}} \cdot \log\left( \frac{td}{\delta} \right)\cdot \log t\right).
    \end{align*}
     If the minibatch size is at least
    \begin{align*}
        b = \Omega\left(\left(\frac{C^2 \log(s)^4} {\kappa^2  \eps^4} \right)^{\frac{1}{\alpha}} \cdot \log\left( \frac{td}{\delta} \right)\right),
    \end{align*}
    then with probability at least $1 - \delta$ over the randomness of $\bS$ and the draws of minibatches from within $\bS$, the size-$t$ decision tree hypothesis constructed by $\MiniBatchTopDown_{\mathscr{G}}(t, b, \bS)$ satisfies $\error_f(T) \le \opt_s + \eps$.
\end{theorem}

\subsection{Properties of batches}

We begin by specifying how large the batch size has to be for accurate estimates of local gain. Later on, we will turn accurate estimates of local gain to estimates of purity gain that are accurate at least half the time.

\begin{lemma}[Every leaf has a batch of size $b_{min}$]
\label{lem:batch-size}
Let 
\[b_{\mathrm{min}} = \max\left(8, 2\cdot \left(\frac{2C}{\Delta}\right)^{\frac{2}{\alpha}}\right) \cdot \log_e\left(\frac{9td}{\delta}\right).\]
Then with probability at least $1-\frac{\delta}{3}$, every leaf $\ell$ satisfying $|\ell| \leq \log(n/(2b_{\mathrm{min}}))$ of the tree that $\MiniBatchTopDown_{\mathscr{G}}(t, b, \bS)$ constructs has a minibatch $\bB\sim\Batch_b(\bS, \ell)$ of size at least $b_{\mathrm{min}}$.
\end{lemma}

\begin{proof}
It suffices to show that the number of points in $\bS$ consistent with each of these $\ell$ is at least $b_{\mathrm{min}}$. Fix any such $\ell$ satisfying $|\ell| \leq \log(n/(2b_{\mathrm{min}}))$. The probability an element in $\bS$ is consistent with $\ell$ is at least $\frac{2b_{\mathrm{min}}}{n}$, meaning the expected number of points consistent with $\ell$ is at least $2b_{\mathrm{min}}$. By the multiplicative Chernoff bound,
    \begin{align*}
        \Pr \left[\sum_{(x,y) \in S} \Ind[\text{$x_i$ consistent with $\ell$}] < b_{\mathrm{min}}\right] \leq \exp_e\left(-\frac{1}{8} \cdot 2b_{\mathrm{min}}\right)
    \end{align*}
    There are at most $t$ leaves that $\MiniBatchTopDown_{\mathscr{G}}(t, b,\bS)$ will ever estimate impurity gain for, so as long as,
    \begin{align*}
        b_{\mathrm{min}} \geq 4 \cdot \log_e\left(\frac{3t}{\delta} \right),
    \end{align*}
    with probability at least $1 - \delta/3$, all of them will have a minibatch of size at least $b_{\mathrm{min}}$.
\end{proof}

\begin{lemma}[Batches are balanced]
\label{lem:batches-balanced}
With probability at least $1-\delta/3$, there are at least $\frac{b_{\mathrm{min}}}{4}$ points $(x,y)$ in $\bB$ satisfying $x_i =-1$ and $\frac{b_{\mathrm{min}}}{4}$ points satisfying $x_i = 1$.
\end{lemma}
\begin{proof}
The mini batch $\bB$ is formed by choosing at least $b_{\mathrm{min}}$ points that are consistent with $\ell$, without replacement, from $\bS$, which is itself formed by taking points with replacement from $\bd$. This means that the mini batch $\bB$ has at least $b_{\mathrm{min}}$ points without replacement from $\bd$. Fix any $\ell$ and let $b_\mathrm{true}$ be the number of points in $\bB$. By Hoeffding's inequality,
    \begin{align*}
        \Pr\left[\left|\frac{b_\mathrm{true}}{2} - (\text{Number of $(x,y) \in \bB$ where $x_i = -1$}) \right| \geq \frac{b_\mathrm{true}}{4} \right] \leq \exp_e(-\frac{b_\mathrm{true}}{8})\\ \leq  \exp_e(-\frac{b_\mathrm{min}}{8})
    \end{align*}
    $\MiniBatchTopDown_{\mathscr{G}}$ computes $\mathrm{LocalGain}_{\mathscr{G},\bB}(\ell, i)$ for at most $t$ different $\ell$ and $d$ different $i$, for a total of $t \cdot d$ different computations. As long as
    \begin{align*}
        b_\mathrm{min} \geq 8 \cdot \log_e\left(\frac{3td}{\delta}\right),
    \end{align*}
    then with probability at least $1-\delta/3$, both $\bB[x_i=-1]$ and $\bB[x_i=1]$ will have at least $\frac{b_\mathrm{true}}{4} \geq \frac{b_\mathrm{min}}{4}$ points.
\end{proof}

\begin{lemma}[Batch size is logarithmic in $td$]
    \label{lem:batch to Delta}
    For any $f: \bits^d \to \{0,1\}$ and $n \in \mathbb{N}$, let $\bS$ be a size $n$ sample of points $(\bx, f(\bx))$ where $\bx \sim \bd$. Furthermore, let $\mathscr{G}:[0,1] \to [0,1]$ be any $(C,\alpha)$-H\"{o}lder continuous impurity function. For any $\Delta > 0$, and
    \begin{align*}
        b \geq b_{\mathrm{min}} = \max\left(8, 2\cdot \left(\frac{2C}{\Delta}\right)^{\frac{2}{\alpha}}\right) \cdot \log_e\left(\frac{9td}{\delta}\right)
    \end{align*}
    with probability at least $1 - \delta$, any time $\MiniBatchTopDown_{\mathscr{G}}(t,  b, \bS)$ computes $\mathrm{LocalGain}_{\mathscr{G},\bB}(\ell, i)$ for $|\ell| \leq \log(n/(2b_{\mathrm{min}}))$ for a mini batch  $\bB \sim \Batch_b(\bS,\ell)$
    \begin{align*}
        |\mathrm{LocalGain}_{\mathscr{G},\bB}(\ell, i) - \localgainf(\ell, i)| \leq \Delta
    \end{align*}
\end{lemma}
\begin{proof}
   For any particular $\ell, i$, in order to compute $\estlocalgainS$ we need to estimate three expectations: 
   \[ \mathscr{G}(\E[f(\bx)], 
   \quad \mathscr{G}(\E[f(\bx)\mid \text{$\bx$ reaches $\ell$, $\bx_i =-1$}\,], 
   \quad 
   \mathscr{G}(\E[f(\bx)\mid \text{$\bx$ reaches $\ell$, $\bx_i =1$}\,].
   \] 
    Define $\eps_1, \eps_2, \eps_3$ to be the errors made in computing these expectations so that
    \begin{align*}
        \estlocalgainS(\ell,i) &\coloneqq \mathscr{G}(\E[\,f(\bx) \mid \text{$\bx$ reaches $\ell$}\,] + \eps_1)  \\ 
      &\quad -   \big(\lfrac1{2} \cdot \mathscr{G}(\E[\,f(\bx)\mid \text{$\bx$ reaches $\ell$, $\bx_i =-1$}\,] + \eps_2) \\
      &\quad\ \,  + \lfrac1{2} \cdot  \mathscr{G}(\E[\,f(\bx)\mid \text{$\bx$ reaches $\ell$, $\bx_i = 1$}\,] + \eps_3)\big).
    \end{align*}
    Suppose that $\eps_1, \eps_2, \eps_3$ are each bounded as
    \begin{align}
        \label{eq:eps small}
        |\eps_j| \leq \left(\frac{\Delta}{2C}\right)^\frac{1}{\alpha}.
    \end{align}
    Then, by the definition of H\"{o}lder continuous and triangle inequality,
    \begin{align*}
        |&\estlocalgainS(\ell,i) - \localgainf(\ell,i)|  \\
        &\leq |\mathscr{G}(\E[\,f(\bx) \mid \text{$\bx$ reaches $\ell$}\,] + \eps_1) - \mathscr{G}(\E[\,f(\bx) \mid \text{$\bx$ reaches $\ell$}\,])| \\
        &\quad + \frac{1}{2} |\mathscr{G}(\E[\,f(\bx)\mid \text{$\bx$ reaches $\ell$, $\bx_i =-1$}\,] + \eps_2) - \mathscr{G}(\E[\,f(\bx)\mid \text{$\bx$ reaches $\ell$, $\bx_i =-1$}\,] )| \\
        &\quad + \frac{1}{2} |\mathscr{G}(\E[\,f(\bx)\mid \text{$\bx$ reaches $\ell$, $\bx_i =1$}\,] + \eps_3) - \mathscr{G}(\E[\,f(\bx)\mid \text{$\bx$ reaches $\ell$, $\bx_i =1$}\,] )| \\
        &\leq C \cdot (|\eps_1|^\alpha + \frac{1}{2} \cdot |\eps_2|^\alpha + \frac{1}{2} \cdot |\eps_3|^\alpha) \\
        &\leq \Delta.
    \end{align*}

    Therefore, it is enough to show that for all $\ell, i$, the corresponding $\eps_1, \eps_2, \eps_3$ satisfy \Cref{eq:eps small}.    By~\Cref{lem:batch-size} and~\Cref{lem:batches-balanced}, with high probability all of these expectations are over at least  $\frac{b_\mathrm{min}}{4}$ terms. Given the above is true, we can use Hoeffding's inequality to bound each $\eps_j$,
    \begin{align*}
        \Pr\left[|\eps_j| > \left(\frac{\Delta}{2C}\right)^\frac{1}{\alpha} \right] \leq \exp_e\left(-2 \cdot \frac{b_\mathrm{min}}{4} \cdot \left(\frac{\Delta}{2C}\right)^{\frac{2}{\alpha}} \right).
    \end{align*}
    There are a total of at most $3td$ such $\eps_j$ we wish to bound. Setting $b_\mathrm{min}$ to at least
    \begin{align*}
        2\cdot \left(\frac{2C}{\Delta}\right)^{\frac{2}{\alpha}} \cdot \log_e\left(\frac{9td}{\delta}\right)
    \end{align*}
    means all are bounded as desired with probability at least $1 - \delta / 3$.
\end{proof}

\subsection{Properties of MiniBatchTopDown}
As we discussed in~\Cref{sec:upper bound mini batch overview}, a key component of~\cite{BLT3}'s analysis is a proof that if $\error_f(T^{\circ}_S) > \opt_s + \eps$, there must exist a leaf $\ell^\star \in T^\circ$ and a coordinate $i^\star\in [d]$ such that 
\begin{equation} \mathrm{PurityGain}_{\mathscr{G},f}(\ell^\star,i^\star)> \frac{\kappa\eps^2}{32j(\log s)^2}.\label{eq:score-lb}\end{equation} 
Based on how we set $\Delta$ in \Cref{lem:batch to Delta}, $\MiniBatchTopDown_{\mathscr{G}}$ will be able to estimate all local gains to additive accuracy $\pm O(\frac{\kappa\eps^2}{\log(s)^2})$. That accuracy, in conjunction with just \Cref{eq:score-lb}, is \emph{not} sufficient to prove that $\MiniBatchTopDown_{\mathscr{G}}$ will produce a low error tree. Instead, we need the following additional fact that \cite{BLT3} proved one step prior to showing \Cref{eq:score-lb}; in fact, it implies \Cref{eq:score-lb} but is stronger, and that strength is needed for our purposes.

\begin{fact}[Showed during the proof of Theorem 2 of \cite{BLT3}]
    \label{fact:total purity gain high}
    Let $T^\circ$ be any partial tree. For any $f: \bd \to \{0,1\}$ and $\kappa$-strongly concave impurity function $\mathscr{G}: [0,1] \to [0,1]$, if $\error_f(T^\circ_S) > \opt_s + \eps$, then 
    \begin{align*}
        \sum_{\text{leaves $\ell \in T^\circ$}} \max_{i \in [d]} \left( \puritygainf(\ell,i) \right) > \frac{\kappa}{32} \cdot \left(\frac{\eps}{\log s}\right)^2.
    \end{align*}
\end{fact}

\Cref{fact:total purity gain high} implies \Cref{eq:score-lb} because, if $T^\circ$ is size $j$ and the total purity gains of all of its leaves is some value $z$, then at least one leaf has purity gain $\frac{z}{j}$. We use \Cref{fact:total purity gain high} to show that, whenever $\MiniBatchTopDown_{\mathscr{G}}$ picks a leaf that is neither too deep nor too high in the tree, it has picked a leaf and index with relatively large purity gain.

\begin{lemma}[Medium depth splits are good.]
    \label{lem:medium splits good}
    Choose any max depth $D \in \mathbb{N}$. Let $f: \bits^d \to \zo$ be a monotone target function and $\mathscr{G}$ be any $\kappa$-strongly concave and $(C, \alpha)$-H\"{o}lder continuous impurity function. For any $s \in \mathbb{N}$, $\eps, \delta \in (0,\frac{1}{2})$, let $t = s^{O(\log(s))/\eps^2}$, and $\bS$ be a set of $n$ labeled examples $(\bx, f(\bx))$ where $\bx \sim \bits^d$ is uniform random, 
    \begin{align*}
        n  = \Omega\left(\left(\frac{C^2 \log(s)^4} {\kappa^2  \eps^4} \right)^{\frac{1}{\alpha}} \cdot \log\left( \frac{td}{\delta} \right)\cdot 2^D\right)
    \end{align*}
    and
    \begin{align*}
        b = \Omega\left(\left(\frac{C^2 \log(s)^4}{\kappa^2  \eps^4}  \right)^{\frac{1}{\alpha}} \cdot \log\left( \frac{td}{\delta} \right)\right).
    \end{align*}    
    With probability at least $1 - \delta$, the following holds for all iterations of $\MiniBatchTopDown_{\mathscr{G}}(t, b, \bS)$. If, at iteration $j$, $T^\circ$ satisfies,
    \begin{align*}
        \error_f(T^\circ_{\Batch_b(S)}) \geq \opt_s + 2\eps,
    \end{align*}
   let $(\ell^\star, i^\star)$ be the leaf and coordinate chosen to maximize the $\estpuritygainS$. Then, if
    \begin{align*}
        \log(j) - 2 \leq |\,\ell^\star\,| \leq D,
    \end{align*}
    then
    \begin{align*}
        \puritygainf(\ell^\star, i^\star) >  \frac{\kappa}{64}\cdot\frac{\eps^2}{j (\log s)^2}.
    \end{align*}
\end{lemma}
\begin{proof}
    %
    %

     For the values of $n$ and $b$ given in this lemma statement, using \Cref{lem:batch to Delta}, we have for $\Delta = \frac{\kappa}{32\cdot10}\cdot (\frac{\epsilon}{\log s})^2$, for all leaves with $|l| \leq D$
           \begin{align}
        |\estlocalgainS(\ell, i) - \localgainf(\ell, i)| \leq \frac{\kappa}{32\cdot10}\cdot \left(\frac{\epsilon}{\log s}\right)^2 \label{eqn:accurate estimates}
    \end{align} 
    with probability atleast $1-\delta$. 
    
   Since $ \error_f((T^\circ)_{\Batch_b(S)}) \geq \opt_s + 2\eps$ and using \cref{lem:batch-size}, we know that $ \error_f((T^\circ)_{S}) \geq \opt_s + \eps$ since the batch size is large enough, we can use \Cref{fact:total purity gain high} to lower bound the estimated purity gain of $\ell^\star$ and $i^\star$. Let $c = \frac{\kappa}{32}$.
    \begin{align*}
        &\sum_{\text{leaves $\ell \in T^\circ$}} \max_{i \in [d]} \left( \puritygainf(\ell,i) \right) > c \cdot \left(\frac{\eps}{\log s}\right)^2 \tag*{(\Cref{fact:total purity gain high})} \\
        &\sum_{\text{leaves $\ell \in T^\circ$}} 2 ^ {-|\ell|} \cdot \max_{i \in [d]} \left( \localgainf(\ell,i) \right) > c \cdot \left(\frac{\eps}{\log s}\right)^2\\
        &\sum_{\text{leaves $\ell \in T^\circ$}} 2 ^ {-|\ell|} \cdot \max_{i \in [d]} \left( \estlocalgainf(\ell,i) \right) > \frac{9c}{10} \cdot \left(\frac{\eps}{\log s}\right)^2 \tag*{(\Cref{eqn:accurate estimates} and $\sum_{\ell} 2^{-|\ell|} = 1$)}\\
         &\sum_{\text{leaves $\ell \in T^\circ$}}  \max_{i \in [d]} \left( \estpuritygainf(\ell,i) \right) > \frac{9c}{10} \cdot \left(\frac{\eps}{\log s}\right)^2. 
    \end{align*}
    Since there are $j$ leaves in $T^\circ$ and $\ell^\star, i^\star$ are chosen to maximize $\estpuritygainf(\ell^\star, i^\star)$,
    \begin{align*}
        \estpuritygainf(\ell^\star, i^\star) > \frac{9c}{10j} \cdot \left(\frac{\eps}{\log s}\right)^2.
    \end{align*}
    Next, we show that since $\ell^*$ is sufficiently far down in the tree, then the estimated purity gain and true purity gain are close.
    \begin{align*}
        |\estpuritygainf(\ell^\star, i^\star) -& \puritygainf(\ell^\star, i^\star)| 
        \\
        &= 2^{-|\ell^\star|} \cdot |\estlocalgainf(\ell^\star, i^\star) - \localgainf(\ell^\star, i^\star)| \\
        &\leq 2^{-|\ell^\star|} \cdot \frac{c}{10} \cdot \left(\frac{\eps}{\log s}\right)^2 \tag*{(\Cref{eqn:accurate estimates})}\\
        &\leq \frac{4}{j} \cdot \frac{c}{10} \cdot \left(\frac{\eps}{\log s}\right)^2. \tag*{($|\ell^\star| \geq \log(j) - 2$)}
    \end{align*}
    By triangle inequality, we have that $\puritygainf(\ell^\star, i^\star) > \frac{c}{2j} \cdot \left(\frac{\eps}{\log s}\right)^2$, the desired result.
    \end{proof}

    %
    

Given that we are only guaranteed to make good progress on splits that are neither too deep nor too shallow, we will need to deal with both possibilities. First, we show that if we ever wanted to make too deep a split, we would already be done. 

\begin{lemma}[Can stop at very large depth.]
    \label{lem:max depth}
    Let $f: \bits^d \to \zo$ be a monotone target function and $\mathscr{G}$ be any $\kappa$-strongly concave and $(C, \alpha)$-H\"{o}lder continuous impurity function. For any $s \in \mathbb{N}$, $\eps, \delta \in (0,\frac{1}{2})$, let 
    \begin{align}
        \label{eq:t definition}
        t = s^{\Theta(\log(s))/(\kappa \eps^2)},
    \end{align}
    set the max depth to
    \begin{align}
        \label{eq:max depth}
        D = \lfloor \log(t) +  \log \log t \rfloor,
    \end{align}
    let $\bS$ be a set of $n$ labeled examples $(\bx, f(\bx))$ where $\bx \sim \bits^d$ is uniform random,
    \begin{align*}
        n = \Omega\left(\left(\frac{C^2 \log(s)^4}{\kappa^2  \eps^4}  \right)^{\frac{1}{\alpha}} \log\left( \frac{td}{\delta} \right) \cdot 2^D \right)  = \poly_{\alpha, \kappa, C}(t, \log(d), \log(1/\delta)),
    \end{align*}
    and batch size at least
    \begin{align*}
        b = \Omega\left(\left(\frac{C^2 \log(s)^4}{\kappa^2  \eps^4}  \right)^{\frac{1}{\alpha}} \log\left( \frac{td}{\delta} \right)\right).
    \end{align*}
    Let $T_1^\circ$, $T_2^\circ, \ldots, T_t^\circ$ be the size $1,2, \ldots, t$ partials trees that $\MiniBatchTopDown_{\mathscr{G}}(t, b, \bS)$ builds. With probability $1 - \delta$ over the randomness of $\bS$ and the random batches, for any $k \in [t]$, if $T_k^\circ$ has depth more than $D$, then
    \begin{align}
        \label{eq:low error}
        \error_f((T_k^\circ)_{\Batch_b(S)}) \leq \opt_s + 2\eps.
    \end{align}

\end{lemma}
\begin{proof}
    
Let $k$ be chosen so that $T_k^\circ$ has depth more than $D$. For some $j \leq k$, there was a leaf $\ell^\star \in T_j^\circ$ that was split, satisfying,
\begin{align*}
    |\ell^\star| &= \lfloor \log(t) +\log \log t \rfloor - 1 \\
    &= \lfloor \log(t) +  \log\left(\Theta\left((\log s)^2 / (\kappa \eps^2)\right)\right) \rfloor - 1.
\end{align*}
For any $i \in [d]$,
\begin{align}
    \label{eq:purity gain limited}
    \puritygainf(\ell^\star, i) &= 2^{-|\ell^\star|} \nonumber \localgainf(\ell^\star, i) \\&\leq 2^{-|\ell^\star|}\nonumber \\
    &\leq \frac{1}{t} \cdot \Theta\left(\left(\frac{(\log s)^2}{\kappa\eps^2} \right)^{-1}\right)\nonumber\\
    &\leq \frac{1}{j} \cdot \Theta\left(\frac{\kappa\eps^2}{ (\log s)^2} \right).
\end{align}
Note that the constant in \Cref{eq:purity gain limited} is inversely related to the constant in the exponent of \Cref{eq:t definition}. In \Cref{lem:medium splits good}, we showed that if $\error_f((T_j^\circ)_{\Batch_b(S)}) \geq \opt_s + \eps$, then for some $i^\star \in [d]$,
\begin{align}
    \label{eq: purity gain high}
    \puritygainf(\ell^\star, i^\star) = \Omega\left(\frac{\kappa \eps^2}{j (\log s)^2} \right).
\end{align}
If we choose the constant in \Cref{eq:purity gain limited} sufficiently low, which can be done by making the constant in \Cref{eq:t definition} sufficiently high, then that equation can not be satisfied at the same time as \Cref{eq: purity gain high}. Therefore, it must be that $\error_f((T_j^\circ)_{\Batch_b(S)}) < \opt_s + \eps$. Since $j < k$, and adding splits can only increase error by atmost $\epsilon$, it must also be the case that $\error_f((T_k^\circ)_{\Batch_b(S)}) < \opt_s + 2\eps$.
\end{proof}

We next show that before \Cref{lem:max depth} kicks in, most splits are sufficiently deep to make good progress.

\begin{lemma}[Few splits are shallow]
    \label{lemma:few splits shallow}
    Let $k = 2^a$ be any power of $2$ and $T_1^\circ,\ldots,T_{k}^\circ$ be a series of bare trees of size $1, \ldots, k$ respectively where $T_{j+1}$ is formed by splitting $\ell_j \in T_j$. Then,
    \begin{align*}
        \sum_{j = 1}^{k} \Ind\big[\,|\ell_j| < \log(j) - 2\,\big] \leq \frac{k}{4}.
    \end{align*}
\end{lemma}
\begin{proof}
    First, since for all $j = 1,\ldots, k$, $j \leq k$, we can bound,
    \begin{align*}
        \sum_{j = 1}^{k} \Ind\big[\,|\ell_j| < \log(j) - 2\,\big] \leq  \sum_{j = 1}^{k} \Ind\big[\,|\ell_j| < \log(k) - 2\,\big] = \sum_{j = 1}^k \Ind\big[\,|\ell_j| < a - 2\,\big].
    \end{align*}
    If $\ell_j$, a leaf of $T_j^\circ$, has depth less than $a - 2$, then it is also an internal node of $T_{k}^\circ$ with depth less than $a - 2$. There are at most $2^{a - 2} - 1$ nodes in any tree of depth less than $a - 2$. Therefore,
    \begin{equation*}
        \sum_{j = 1}^k \Ind[\,|\ell_j| \leq a - 2\,] \leq 2^{a - 2} - 1 \leq \frac{k}{4} - 1 \leq \frac{k}{4}. \qedhere
    \end{equation*} 
\end{proof}

\begin{proof}[Proof of~\Cref{thm:upper bound mini batch formal}]

$\MiniBatchTopDown_{\mathscr{G}}$ builds a series of bare trees, $T_1^\circ, T_2^\circ, \ldots, T_{t}^\circ$, where $T_j$ has size $j$. We wish to prove that $\error_f((T_{t}^\circ)_{\Batch_b(S)}) \leq \opt_s + 3\eps$ (In the end, we can choose $\eps$ appropriately to get error $\opt_s + \eps$). To do so, we consider two cases.\medskip 

\noindent \textbf{Case 1}: There is some $k < t$ for which $\error_f((T_{k}^\circ)_{\Batch_b(S)}) \leq \opt_s + 2\eps$.\\
Since splitting more variables of $T_k$ can only increase it's error by at most $\epsilon$,
\begin{align*}
    \error_f((T_{t}^\circ)_{\Batch_b(S)}) \leq \error_f( (T_{k}^\circ)_{\Batch_b(S)}) + \eps \leq \opt_s + 3\eps,
\end{align*}
which is the desired result.

\medskip 
\noindent \textbf{Case 2}: There is no $k < {t}$ for which $\error_f((T_{k}^\circ)_{\Batch_b(S)}) \leq \opt_s + \eps$. \\
In this case, we use \Cref{lem:medium splits good} to ensure we make good progress. \Cref{lem:medium splits good} only applies when the tree has depth at most $D = \log t + \log \log t$. Luckily, \Cref{lem:max depth} ensures that if the tree has depth more than $D$, then we are ensured that $\error_f((T_{t}^\circ)_{\Batch_b(S)}) \leq \opt_s + 2\eps$, and so are done. For the remainder of this proof, we assume all partial trees have depth at most $D$.

We will show that $\Gimpurity(T_{t}^\circ) = 0$, which means that $\error_f((T_{t}^\circ)_{\Batch_b(S)}) = 0 \leq \opt_s + 2\eps$, also proving the desired result. For $j = 1,\ldots, t - 1$, let $\ell_j$ be the leaf of $T_j$ that is split, and $i_j$ be the coordinate placed at $\ell_j$ to form $T_{j+1}$. Then,
\begin{align*}
    \Gimpurity(T_{t}^\circ) = \Gimpurity(T_{1}^\circ) - \sum_{j = 1}^{t-1} \puritygainf(\ell_j, i_j) .
\end{align*}
Since $\Gimpurity(T_{1}^\circ) \leq 1$ and our goal is to show that $\Gimpurity(T_{t+1}^\circ) = 0$, it is sufficient to show that $\sum_{j = 1}^t \puritygainf(\ell_j, i_j) \geq 1$. \Cref{lem:medium splits good} combined with \Cref{fact:total purity gain high},
\begin{align*}
    \sum_{j = 1}^t \puritygainf(\ell_j, i_j) \geq  \sum_{j = 1}^t \Ind[\,|\ell_j| \geq \log(j) - 2\,] \cdot \frac{\kappa}{64j} \cdot \left(\frac{\eps}{\log s}\right)^2.
\end{align*}
We break the above summation into chunks from $j = (2^{a} + 1)$ to $j = 2^{a+1}$, integer $a \leq \log(t)$. In such a chunk, there are $2^a$ choices for $j$. By \Cref{lemma:few splits shallow}, we know that for at most $2^{a+1} / 4 = 2^a / 2$ of those $j$ is $\Ind[\,|\ell_j| < \log(j) - 2\,]$. Therefore,

\begin{align*}
    \sum_{j = 2^{a} + 1}^{2^{a+1}} \Ind[\,|\ell_j| \leq \log(j) - 2\,] \cdot \frac{\kappa}{64j} \cdot \left(\frac{\eps}{\log s}\right)^2 &\geq \frac{2^a}{2} \cdot \frac{\kappa}{64 \cdot (2^{a+1})} \cdot \left(\frac{\eps}{\log s}\right)^2 \\
    &=\frac{\kappa}{256} \cdot \left(\frac{\eps}{\log s}\right)^2.
\end{align*}
Summing up $\frac{256}{\kappa} \cdot \left(\frac{\log s}{\eps}\right)^2$ such chunks gives a sum of at least $1$. Therefore, for
\begin{align*}
    t = \exp \left(\Omega\left(\frac{(\log s)^2}{\kappa \eps^2}  \right) \right)
\end{align*}
it must be the case that $\Gimpurity(T_{t+1}^\circ) = 0$, proving the desired result.
\end{proof}
%
%
\section{Proofs of \Cref{thm:local learner,thm:estimate-learn}}
We begin with a proof overview for~\Cref{thm:local learner}.  Let $T$ be the decision tree hypothesis that $\MiniBatchTopDown_{\mathscr{G}}$ would construct if we were to all of $S^\circ$ and train $\MiniBatchTopDown_{\mathscr{G}}$ on it. Our goal is to efficiently compute $T(x^\star)$ for a given $x^\star$ by selectively labeling only $q$ points within $S^\circ$, where $q$ is exponentially smaller than the sample complexity of learning and constructing~$T$. 

Intuitively, we would like $\LocalLearner_{\mathscr{G}}$ to only grow the single ``strand'' within $T$ required to compute $T(x^\star)$ instead of the entire tree $T$---this ``strand'' is simply the root-to-leaf path of $T$ that $x^\star$ follows.  The key challenge that arises in implementing this plan is: how does $\LocalLearner_{\mathscr{G}}$ know when to terminate this strand (i.e.~how does it know when it has reached a leaf of $T$)?  $\MiniBatchTopDown_{\mathscr{G}}$, the ``global'' algorithm that $\LocalLearner_{\mathscr{G}}$ is trying the simulate, terminates when the tree is of size $t$.  As $\LocalLearner_{\mathscr{G}}$ grows the strand corresponding to $x^\star$, how could it estimate the size of the overall tree without actually growing it?  In other words, it is not clear how one would define the stopping criterion of the while loop in the following pseudocode: 

\begin{figure}[H]
  \captionsetup{width=.9\linewidth}
\begin{tcolorbox}[colback = white,arc=1mm, boxrule=0.25mm,left=20pt]

 Initialize $\ell$ to be the leaf of the empty tree. \vspace{4pt} 
 
  while (\textit{stopping criterion}) \{  \\

\vspace{-12pt}
\begin{enumerate}
\item Draw $\bB^\circ \sim \Batch_b(S^\circ,\ell)$ and query $f$'s values on these points.  Let $i^\star$ be the coordinate that maximizes $\mathrm{PurityGain}_{\mathscr{G},\bB^\circ}(\ell, i)$ among all $i\in [d]$.
\item Extend $\ell$ according to the value of $x^\star_{i^\star}$. 

\end{enumerate}
\vspace{-8pt}

\ \ \}  \vspace{4pt}

 Draw $\bB^\circ \sim \Batch_b(S,\ell)$ and query $f$'s values on these points. \vspace{4pt} \\
 Output $\mathrm{round}(\E_{\bx\sim\bB^\circ}[f(\bx)])$.

\end{tcolorbox}
\label{fig:informal single strand local learner}
\end{figure}
Roughly speaking, we want ``\textit{stopping criterion}'' to answer the following question: if we grew a size-$t$ tree using $\MiniBatchTopDown_{\mathscr{G}}$ (on the labeled version of $S^\circ$), would $\ell$ be a leaf of the resulting tree, or would it be an internal node? Nearly equivalently, with access to just a single strand of a tree, we wish to estimate the size of that tree. If that size is $t$, then we stop the while loop. 

It is not possible to accurately estimate the size of a tree using just a single strand. However, by computing a small number of random strands, we can get an accurate size estimator. In~\Cref{section:size estimator}, we show that for $\bx_1, \ldots, \bx_m$ chosen uniformly at random from $\bits^d$, the estimator
$    \boldsymbol{e} \coloneqq \frac{1}{m} \sum_{i=1}^m 2^{|\ell_T(\bx_i)|}
$
accurately estimates the size of  $T$, as long as the depth of $T$ is not too large. Therefore, rather than growing only the root-to-leaf path for $x^\star$, $\LocalLearner_{\mathscr{G}}$ samples random additional inputs, $\bx_1, \ldots, \bx_m$. Then, it simultaneously grows the strands for the root-to-leaf paths of $x^\star$ as well as $\bx_1, \ldots, \bx_m$. These strands do not all grow at the same ``rate'', as we want $\LocalLearner_{\mathscr{G}}$ to make splits in the same order as $\MiniBatchTopDown_{\mathscr{G}}$ does. As long as it does this, we can use the size estimator to, at any step, accurately estimate the size of tree $\MiniBatchTopDown_{\mathscr{G}}$ would need to build for all the current strands to end at leaves. $\LocalLearner_{\mathscr{G}}$ terminates when its estimate of this size is $t$.

\begin{figure}[h!]
\centering
  \includegraphics[width=0.6\textwidth, angle=0]{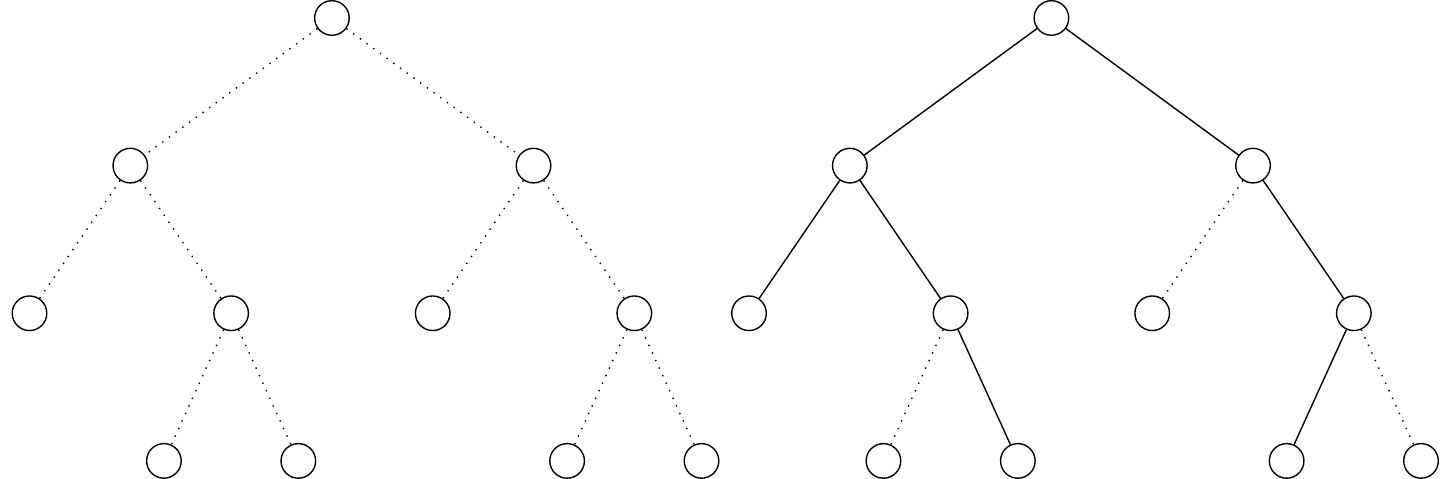}
    \caption{Rather than growing the entire tree $T$ (depcited on the LHS) as $\MiniBatchTopDown_{\mathscr{G}}$ does, $\LocalLearner_{\mathscr{G}}$ only grows $m+1$ strands within $T$ (depicted on the RHS), corresponding to the given input $x^\star$ and $m$ additional  random inputs $\bx_1,...,\bx_m \sim\bits^d$.}
\end{figure}

We back the above intuition for $\LocalLearner_{\mathscr{G}}$ with proofs.
In~\Cref{sec:local-learner}, 
we show that the output of $\LocalLearner_{\mathscr{G}}$ for size parameter $t$  is $T(x^\star)$, where $T$ is size-$t'$ tree produced by $\MiniBatchTopDown_{\mathscr{G}}$ where $t \in t'(1\pm \eta)$.  We also show that $\LocalLearner_{\mathscr{G}}$ needs to only label polylogarithmic many points within $S^\circ$ to compute $T(x^*)$. This completes our proof overview for~\Cref{thm:local learner}, and \Cref{thm:estimate-learn} is a straightforward consequence of~\Cref{thm:local learner}. 

\subsection{Estimating the size of a decision tree}
\label{section:size estimator}

In this section, we design a decision tree size estimator. This size estimator only needs to inspect a small number of random strands from the decision tree. It is unbiased, and as long as the decision tree has a bounded max depth, obeys concentration bounds shown in \Cref{lemma:size estimator}.

\begin{lemma}[Size estimator]
    \label{lemma:size estimator}
    For any $\Delta, \delta > 0$ and size-$s$ decision tree $T$, let $\ell^\star$ be the deepest leaf in $T$ and
    \begin{align*}
        m = \frac{(2^{|\ell^\star|})^2}{2 \Delta^2} \cdot \ln\left(\frac{2}{\delta}\right).
    \end{align*}
    Choose $\bx_1,\ldots, \bx_m$ uniformly random from $\bd$ and define the estimator
    \begin{align*}
        e \coloneqq \frac{1}{m} \sum_{i=1}^m 2^{|\ell_{T}(\bx_i)|}.
    \end{align*}
    With probability at least $1 - \delta$,
    \begin{align*}
        |e - s| \leq \Delta.
    \end{align*}
\end{lemma}
\begin{proof}
    We first show that $\Ex[e] = s$. 
    \begin{align*}
        \Ex[e] &= \Ex_{\bx \sim \bd} \left[2^{|\ell_T(\bx)|}\right] \\
        &= \sum_{\text{leaves $\ell \in T$}} \Pr[\text{$\bx$ reaches $\ell$}] \cdot 2^{|\ell|} \\
        &= \sum_{\text{leaves $\ell \in T$}} \frac{1}{2^{|\ell|}} \cdot 2^{|\ell|} \\
        &= s,
    \end{align*}
    where the last equality is due to the fact that a size-$s$ tree has $s$ leaves. Furthermore, $e$ is the sum of $m$ independent random variables bounded between $0$ and $2^{|\ell^\star|}$. Therefore, we can apply Hoeffding's inequality,
    \begin{align*}
        \Pr[|e - s| \geq \Delta] \leq 2\exp_e\left(-\frac{2 m \Delta^2}{(2^{|\ell^\star|})^2}\right).
    \end{align*}
    Plugging in $m$ proves the desired result.
\end{proof}
\subsection{Provable guarantees for $\LocalLearner$}
\label{sec:local-learner}
To facilitate comparisons between the output of $\LocalLearner_{\mathscr{G}}$ and $\MiniBatchTopDown_{\mathscr{G}}$, we will define another algorithm, $\TopDownSizeEstimate_{\mathscr{G}}$ (\Cref{fig:TopDownSizeEstimate}), that shares some elements with $\LocalLearner_{\mathscr{G}}$ and some elements with $\MiniBatchTopDown_{\mathscr{G}}$.

\begin{figure}[htb!]
  \captionsetup{width=.9\linewidth}
\begin{tcolorbox}[colback = white,arc=1mm, boxrule=0.25mm,left=20pt]
\vspace{3pt} 

\hspace{-10pt}$\TopDownSizeEstimate_{\mathscr{G}}$($t,b,S$):  \vspace{6pt} 

Initialize $T^\circ$ to be the empty tree. \vspace{4pt} 
 
Define $D \coloneqq \log t + \log\log t$.\vspace{4pt}

Let $\bB^\circ_{\mathrm{strands}}$ be $b$ uniform random points from $\bits^d$. \vspace{4pt} 

Initialize $e \coloneqq 1$, our size estimate. \vspace{4pt} 

 while ($e < t$) \{  \\

\vspace{-10pt}
\begin{enumerate}[leftmargin=20pt]
\item \hspace{-5pt} {\sl Score:}  For each leaf $\ell \in T^\circ$ of depth at most $D$, draw $\bB \sim \Batch_b(S,\ell)$. For each coordinate $i\in [d]$,  compute:
\begin{align*}  
\mathrm{PurityGain}_{\mathscr{G},\bB}(\ell,i) &\coloneqq 2^{-|\ell|} \cdot \mathrm{LocalGain}_{\mathscr{G},\bB} (\ell,i), \text{ where}\\
  \mathrm{LocalGain}_{\mathscr{G},\bB}(\ell,i) &\coloneqq \mathscr{G}(\E[f(\bx)])  \\ 
  &\quad -   \big(\lfrac1{2} \, \mathscr{G}(\E[\,f(\bx)\mid \text{$\bx_i =-1$}\,]) + \lfrac1{2} \, \mathscr{G}(\E[\,f(\bx)\mid \text{$\bx_i = 1$}\,])\big),
\end{align*}
where the expectations are with respect to $(\bx,f(\bx))\sim \bB$.
\item {\sl Split:} Let $(\ell^\star,i^\star)$ be the tuple that maximizes $\mathrm{PurityGain}_{\mathscr{G},\bB}(\ell,i)$.  Grow $T^\circ$ by splitting $\ell^\star$ with a query to $x_{i^\star}$.
\item {\sl Estimate size:} Update our size estimate to
\begin{align*}
    e = \Ex_{\bx \in \bB^\circ_{\mathrm{strands}}}[2^{|\ell_{T^{\circ}}(\bx)|}]
\end{align*}
\end{enumerate}
\vspace{-18pt}

\ \ \}  \vspace{4pt}

For each leaf $\ell \in T^\circ$, draw $\bB \sim \Batch_b(S,\ell)$ and label $\ell$ with $\mathrm{round}(\E_{(\bx,f(\bx))\sim\bB}[f(\bx)])$.

\vspace{3pt}
\end{tcolorbox}
\caption{$\TopDownSizeEstimate_{\mathscr{G}}$ takes as input a size parameter $t$, a minibatch size $b$, and a labeled dataset $S$.  It outputs a size-$t'$ decision tree hypothesis for $f$, where $t'$ is close to $t$.}
\label{fig:TopDownSizeEstimate}
\end{figure}

\paragraph{Comparison between $\MiniBatchTopDown_{\mathscr{G}}$ and $\TopDownSizeEstimate_{\mathscr{G}}$:} The only difference between $\MiniBatchTopDown_{\mathscr{G}}$ and $\TopDownSizeEstimate_{\mathscr{G}}$ is the stopping criterion. $\MiniBatchTopDown_{\mathscr{G}}$ stops when the size of $T^\circ$ is exactly $t$, whereas $\TopDownSizeEstimate_{\mathscr{G}}$ estimates the size of $T^\circ$ using the estimator from \Cref{section:size estimator} and stops when this size estimate is at least $t$.

\paragraph{Comparison between $\LocalLearner_{\mathscr{G}}$ and $\MiniBatchTopDown_{\mathscr{G}}$:} For any $t, S, b, x^\star$ that are valid inputs to $\LocalLearner_{\mathscr{G}}$, we compare the following two procedures.
\begin{enumerate}
    \item Running $\TopDownSizeEstimate_{\mathscr{G}}(t, b, S)$ to get a decision tree, $T$, and then computing $T(x^\star)$.
    \item Only running $\LocalLearner_{\mathscr{G}}(t, b, S, x^\star)$.
\end{enumerate}
We claim the output from the above two procedures is identical (given \Cref{footnote:assumption}). To see this, we first observe that $\TopDownSizeEstimate_{\mathscr{G}}$ expands all paths in the tree its building, whereas $\LocalLearner_{\mathscr{G}}$ only expands paths that are pertinent to either the input $x^\star$, or inputs in $\bB_{\mathrm{strands}}^\circ$, which are used to compute the size estimate. Aside from that, both of the above procedures are identical. Furthermore, paths not containing $x^\star$ nor any inputs in $\bB_{\mathrm{strands}}^\circ$ have no effect on how the tree eventually labels $x^\star$. Therefore, the output of the two above procedures is identical, though $\LocalLearner_{\mathscr{G}}$ is more efficient as it only computes necessary paths.

Combining the above observations, we are able to prove the formal version of \Cref{thm:local learner}.
\begin{theorem}[Formal version of \Cref{thm:local learner}]
\label{thm:local learner formal}
    Let $f: \bits^d \to \zo$ be a target function, $\mathscr{G}$ be an impurity function, and  $S^\circ$ be an {\sl unlabeled} training set. For all $t \in \mathbb{N}$ and $\eta, \delta \in (0,\frac{1}{2})$, if the minibatch size is at least
    \begin{align*}
        b = \Omega\left(\frac{(\log t)^2}{\eta^2} \cdot \log\left(\frac{t}{\delta}\right) \right),
    \end{align*}
    then with probability at least $1 - \delta$ over the randomness of $\bB_{\mathrm{strands}}^\circ$, there is some $t' \in [t - \eta t, t + \eta t]$ for which the following holds. For all $x^\star\in\bits^d$, $\LocalLearner_{\mathscr{G}}(t,b,S^\circ,x^\star)$ labels 
  \[ q = O(b^2\log t) \]  points within $S^\circ$ and returns $T(x^\star)$, where $T$ is the size-$t'$ decision tree hypothesis that \[ \MiniBatchTopDown_{\mathscr{G}}(t',b,S) \]  would construct, and $S$ is the labeled dataset obtained by labeling all of $S^\circ$ with $f$'s values.

\end{theorem}

We break the proof of \Cref{thm:local learner formal} into two pieces. First, we show that it labels only $O(b^2 \log t)$ points within $S^\circ$, and then the rest.

\begin{lemma}[Label efficiency of $\LocalLearner_{\mathscr{G}}$]
    \label{lem:local learner few labels}
    Let $f: \bits^d \to \zo$ be a target function, $\mathscr{G}$ be an impurity function, and  $S^\circ$ be an {\sl unlabeled} training set. For any $b, t \in \mathbb{N}$ and $x^\star \in \bits^d$, $\LocalLearner_{\mathscr{G}}(t, b, S^\circ, x^\star)$ labels at most
  \[ q = O(b^2\log t) \]  points within $S^\circ$.
\end{lemma}
\begin{proof}
     It is sufficient for us to show that $\LocalLearner_\mathscr{G}$ labels at most $O(b \log t)$ batches.  $\LocalLearner_\mathscr{G}$ builds a series of bare trees $T_1^\circ, \ldots, T_{t'}^\circ$. During the while loop, the number of batches it labels is equal to nodes in the following set
     \begin{align*}
         L \coloneqq \bigcup_{j=1}^{t'}\left\{\ell_{T_j^\circ}(x)\colon x \in \bB^\circ_{\mathrm{strands}}\cup \{x^\star\}, |\ell_{T_j^\circ}(x)| \leq D \right\}
     \end{align*}
     Consider a single $x \in \bB^\circ_{\mathrm{strands}}\cup \{x^\star\}$, and define
     \begin{align*}
         L(x) \coloneqq \left\{\ell_{T_j^\circ}(x)\, \colon \, j \in [t'], |\ell_{T_j^\circ}(x)| \leq D \right\}.
     \end{align*}
     Every node in $L(x)$ has depth at most $D$, and there is at most one node in $L(x)$ per depth. Therefore, $|L(x)| \leq  D$, and
     \begin{align*}
         |L| &\leq \sum_{x \in \bB^\circ_{\mathrm{strands}}\cup \{x^\star\}} |L(x)| \\
         &\leq (b + 1) D\\
         &= O(b \log t).
     \end{align*}
     Therefore, $\LocalLearner_{\mathscr{G}}$ labels only $O(b \log t)$ batches during the while loop. After the while loop, it labels at most $1$ additional batches. Therefore, it labels a total of $O(b \log t)$ batches which requires labeling $O(b^2 \log t)$ points.
\end{proof}

We next prove the remainder of \Cref{thm:local learner formal}.

\begin{proof}
    Let $T$ be the tree that $\TopDownSizeEstimate_\mathscr{G}(t, b, S)$ produces. In the comparison between $\LocalLearner_\mathscr{G}$ and $\TopDownSizeEstimate_\mathscr{G}$, we established that, for all $x^\star \in \bits^d$,
    \begin{align*}
        \LocalLearner_\mathscr{G}(t, b, S, x^\star) = T(x^\star).
    \end{align*}
    Set $t' = |T|$. Then, $T$ is also the output of $\MiniBatchTopDown_\mathscr{G}(t', b, \bS)$, as desired. Next, we prove that $t' \in [t - \eta t, t + \eta t]$ with probability at least $1 - \delta$.
    
   Let $T_1^\circ, T_2^\circ, \ldots, T_{t'}^\circ$ be the bare trees of size $1, 2, \ldots, t'$ that $\TopDownSizeEstimate_\mathscr{G}(t, b, \bS)$ produces, and let $e_1, e_2, \ldots e_{t'}$ be the corresponding size estimates. Since $\TopDownSizeEstimate_{\mathscr{G}}$ halts when the size estimate is at least $t$,
    \begin{align*}
        e_{t'} \geq t \quad \text{and} \quad
        e_{t'-1} < t.
    \end{align*}
    We set $\Delta \coloneqq \eta t$ and wish, for all $1 \leq j \leq t + \eta t$, that $e_j$ estimate the size of $T_j^\circ$ to accuracy $\pm \Delta$. Since the size of $T_j^\circ$ is $j$, we equivalently wish for
    \begin{align}
        \label{eq:accurate size all}
        |e_j - j| \leq \Delta \quad \text{for all $j = 1,\ldots, t + \eta t$}.
    \end{align}
    Each $T_j^\circ$ has max depth at most $\log t + \log\log t$. By \Cref{lemma:size estimator} and a union bound over all $t + \eta t$ different $j$, we can guarantee that \Cref{eq:accurate size all} holds with probability at least $1 - \delta$ if we set
    \begin{align*}
        b &\geq \frac{(2^{\log t + \log \log t})^2}{2 \Delta^2} \cdot \ln\left(\frac{2t(1 + \eta)}{\delta}\right) \\
        &=\Omega\left(\frac{(t \log t)^2}{(\eta t)^2} \cdot \log\left(\frac{t}{\delta}\right) \right) \\
        &= \Omega\left(\frac{(\log t)^2}{\eta^2} \cdot \log\left(\frac{t}{\delta}\right) \right).
    \end{align*}
    Therefore, for the $b$ we set in \Cref{thm:local learner formal}, \Cref{eq:accurate size all} holds with probability at least $1 - \delta$.  For the remainder of this proof, we suppose it holds and then show that the $t' \in [t - \eta t, t + \eta t]$. We first show that $t' \leq t + \eta t$. By \Cref{eq:accurate size all}, for $j = t + \eta t$,
    \begin{align*}
        e_{(t + \eta t)} &\geq (t + \eta t) - \Delta \\
        &\geq t.
    \end{align*}
    Recall that $t'$ is the lowest integer such that $e_{t'} \geq t$. Therefore, $t' \leq t + \eta t$. We next show that $t' \geq t - \eta t$. By \Cref{eq:accurate size all} for $j = t' \leq t + \eta t$,
    \begin{align*}
        t' &\geq e_{t'} - \Delta\\
        &\geq t - \eta t.
    \end{align*}
    Therefore, \Cref{eq:accurate size all} implies $t' \in [t - \eta t, t + \eta t]$ proving that with probability at least $1 - \delta$.
    \end{proof}

Finally, we show that the following algorithm estimates learnability.

\begin{figure}[H]
  \captionsetup{width=.9\linewidth}
\begin{tcolorbox}[colback = white,arc=1mm, boxrule=0.25mm,left=20pt]
\vspace{3pt} 

\hspace{-10pt}$\Est_{\mathscr{G}}$($t,b,S^\circ, S_{\mathrm{test}}$):  \vspace{6pt} 

Return
\[     \frac{1}{|S_{\mathrm{test}}|} \sum_{(x,y) \in S_{\mathrm{test}}} \Ind\big[\LocalLearner_{\mathscr{G}}(t, b, S^\circ, x) \neq y\big]
\] 

\vspace{3pt}
\end{tcolorbox}
\label{fig:Est}
\caption{$\Est_{\mathscr{G}}$ takes as input a size parameter $t$, a minibatch size $b$, and an unlabeled dataset $S^\circ$ and labeled test set $S_{\mathrm{test}}$. It outputs the error of the tree returned by $\MiniBatchTopDown(t', b, S)$ with respect to $S_{\mathrm{test}}$, where $S$ is the labeled version of $S^\circ$ and $t'$ is close to $t$. As in \Cref{footnote:assumption}, the random outcome of $\bB_{\mathrm{strands}}^\circ$ and the minibatches should be consistent across all runs of $\LocalLearner_{\mathscr{G}}$.}
\end{figure}

\begin{theorem}[Formal version of \Cref{thm:estimate-learn}] 
\label{thm:estimate-learn-formal}
    Let $f: \bits^d \to \zo$ be a target function, $\mathscr{G}$ be an impurity function,  $S^\circ$ be an {\sl unlabeled} training set, and $S_{\mathrm{test}}$ be a labeled test set. 

    For all $t \in \mathbb{N}$ and $\eta, \delta \in (0,\frac{1}{2})$, if the minibatch size $b$ is as in~\Cref{thm:local learner formal}, then         
    with probability at least $1 - \delta$ over the randomness of $\bB_{\mathrm{strands}}^\circ$, $\Est_{\mathscr{G}}(t,b,S^\circ,S_{\mathrm{test}})$ labels 
    \[ q = O(|S_{\mathrm{test}}|\cdot b \log t + b^2\log t)  \]  
    points within $S^\circ$ and returns 
    \[ \error_{S_{\mathrm{test}}}(T) \coloneqq \Prx_{(\bx,\by)\sim S_{\mathrm{test}}}[T(\bx) \ne \by], \] 
where $T$ is as in~\Cref{thm:local learner formal}. 
\end{theorem} 
\begin{proof}
    Based on \Cref{thm:local learner formal}, $\Est_{\mathscr{G}}$ returns the desired result, so we only need to prove it labels few points within $S^\circ$. As in \Cref{footnote:assumption}, the same $\bB^\circ_{\mathrm{strands}}$ are chosen across multiple runs of $\LocalLearner_{\mathscr{G}}$. As shown in \Cref{lem:local learner few labels}, the total number of points it labels is $O(b \log t) * m$, where $m$ is the number of strands built. $\Est_{\mathscr{G}}$ needs to build $b$ strands for points within $\bB^\circ_{\mathrm{strands}}$ and $|S_{\mathrm{test}}|$ strands for the points within $S_{\mathrm{test}}$. As long as it caches its labels across runs of $\LocalLearner_{\mathscr{G}}$, the total labels used will be
    \begin{equation*}
        q = O(b \log t) \cdot (b + |S_{\mathrm{test}}|) = O(|S_{\mathrm{test}}|\cdot b \log t + b^2\log t).\qedhere
    \end{equation*}
\end{proof}

\section{Conclusion} 

We have given strengthened provable guarantees on the performance of popular and empirically successful top-down decision tree learning heuristics such as ID3, C4.5, and CART, focusing on sample complexity.  First, we designed and analyzed {\sl minibatch} versions of these heuristics, $\MiniBatchTopDown_{\mathscr{G}}$, and proved that they achieve the same performance guarantees as the full-batch versions.  We then gave an implementation of $\MiniBatchTopDown_{\mathscr{G}}$ within the recently-introduced active local learning framework of~\cite{BBG20}.  Building on these results, we showed that $\MiniBatchTopDown_{\mathscr{G}}$ is amenable to highly efficient  learnability estimation~\cite{KV18,BH18}:  its performance can be estimated accurately by selectively labeling very few  examples. 

As discussed in~\cite{KV18,BH18}, this new notion of learnability estimation opens up a whole host of theoretical and empirical directions for future work.  We discuss several  that are most relevant to our work:  
\begin{itemize}[leftmargin=10pt] 
\item[$\circ$] Our algorithm $\mathrm{Est}_{\mathscr{G}}$ efficiently and accurately estimates the quality, relative to a test set $S_{\mathrm{test}}$, of the hypothesis that $\MiniBatchTopDown_{\mathscr{G}}$ would produce if trained on a set $S^\circ$.  Could $\mathrm{Est}_{\mathscr{G}}$ be more broadly useful in assessing the quality of the {\sl training data $S^\circ$ itself}, relative to $S_{\mathrm{test}}$?  Could its estimates provide guarantees on the performance of other algorithms when trained in $S^\circ$ and tested on $S_{\mathrm{test}}$? 
\item[$\circ$] It would be interesting to explore applications of our algorithms to the design of training sets. Given training sets $S_1,\ldots,S_m$, the procedure $\mathrm{Est}_{\mathscr{G}}$ allows us to efficiently determine the $S_i$ for which $\MiniBatchTopDown_{\mathscr{G}}$ would produce a hypothesis that achieves the smallest error with respect to $S_{\mathrm{test}}$.  Could $\mathrm{Est}_{\mathscr{G}}$ or extensions of it be useful in efficiently {\sl creating an $S^\star$}, comprising data from each $S_i$, that is of higher quality than any $S_i$ individually? 
\item[$\circ$] Finally, while we have focused on top-down heuristics for learning a single decision tree in this work, a natural next step would be to design and analyze learnability estimation procedures for ensemble methods such as random forests and gradient boosted trees. 
\end{itemize} 

\section*{Acknowledgements}

We thank the NeurIPS reviewers for their thoughtful and valuable feedback. 

Guy, Jane, and Li-Yang were supported by NSF award CCF-1921795 and NSF CAREER award CCF-1942123. Neha was supported by NSF award 1704417 and Moses Charikar’s Simons Investigator grant.
\bibliography{most-influential}{}
\bibliographystyle{alpha}
\appendix
\end{document}